%% file: main.tex
\documentclass{article}

\usepackage{arxiv}

\usepackage[utf8]{inputenc}
\usepackage[T1]{fontenc}
\usepackage{natbib}
\setcitestyle{authoryear,round,citesep={;},aysep={,},yysep={;}}
\usepackage{hyperref}
\usepackage{url}
\usepackage{booktabs}
\usepackage{amsfonts}
\usepackage{nicefrac}
\usepackage{microtype}
\usepackage{graphicx}
\usepackage{wrapfig}
\usepackage{placeins}
\usepackage{doi}
\usepackage{amsmath,amssymb,bm}
\usepackage{amsthm}
\usepackage{xcolor}
\usepackage{colortbl}
\definecolor{groupbg}{RGB}{235,242,246}
\definecolor{grouptext}{RGB}{101,132,150}
\usepackage{enumitem}
\usepackage{cleveref}
\usepackage{etoolbox}

\makeatletter
\patchcmd{\@maketitle}
    {\vskip 0.4in \@minus 0.1in \center{\@date} \vskip 0.2in}
    {\vskip 0pt}
    {}
    {\PackageError{convstack}{Could not adjust title spacing}{Check arxiv.sty.}}
\makeatother

\theoremstyle{plain}
\newtheorem{proposition}{Proposition}
\newtheorem{theorem}{Theorem}
\theoremstyle{definition}
\newtheorem{assumption}{Assumption}
\crefname{proposition}{proposition}{propositions}
\Crefname{proposition}{Proposition}{Propositions}
\crefname{theorem}{theorem}{theorems}
\Crefname{theorem}{Theorem}{Theorems}
\crefname{assumption}{assumption}{assumptions}
\Crefname{assumption}{Assumption}{Assumptions}

\hypersetup{
    colorlinks=true,
    linkcolor=blue,
    citecolor=blue,
    urlcolor=blue
}

\graphicspath{{Figures/}}

\newcommand{\methodname}{ConvStack}

\newcommand{\synpoly}{SynPoly}
\newcommand{\syndot}{SynDot}

\title{Why MLLMs Struggle to Count: Overcoming Individuation and Aggregation Bottlenecks with \methodname{}}

\author{
    \textbf{Liwei~Che\textsuperscript{1,*} \quad
    Yihao~Quan\textsuperscript{1,*} \quad
    Sen~Fang\textsuperscript{1} \quad
    Hongyi~Wang\textsuperscript{1}}\\
    \textbf{Ranjay~Krishna\textsuperscript{2} \quad
    Ruixiang~Tang\textsuperscript{1,\textdagger} \quad
    Vladimir~Pavlovic\textsuperscript{1,\textdagger}}\\[0.1em]
    \normalfont\textsuperscript{1}Department of Computer Science, Rutgers University\\
    \normalfont\textsuperscript{2}Department of Computer Science, University of Washington\\
    {\normalfont\small\textsuperscript{*}Equal contribution. \quad
    \textsuperscript{\textdagger}Co-corresponding authors.}
}
\date{}

\renewcommand{\shorttitle}{ConvStack: Overcoming Visual Counting Bottlenecks}

\hypersetup{
    pdftitle={Why MLLMs Struggle to Count: Overcoming Individuation and Aggregation Bottlenecks with ConvStack},
    pdfsubject={Computer Vision, Multimodal Large Language Models, Visual Counting},
    pdfauthor={Liwei Che, Yihao Quan, Sen Fang, Hongyi Wang, Ranjay Krishna, Ruixiang Tang, Vladimir Pavlovic},
    pdfkeywords={visual counting, vision transformers, spatial grounding, multimodal large language models}
}

\begin{document}
\maketitle

\input{Sections/abstract}

\input{Sections/intro}
\input{Sections/related_work}
\input{Sections/method}
\input{Sections/exp}
\input{Sections/conclusion}

\begingroup
\setlength{\bibsep}{2pt plus 0.3pt minus 0.3pt}
\bibliographystyle{unsrtnat}
\bibliography{references}
\endgroup

\clearpage
\appendix
\numberwithin{equation}{section}
\section*{Appendix}
\input{Sections/appendix_datasets}
\input{Sections/appendix_method}
\input{Sections/appendix_ablation}
\input{Sections/appendix_count_geometry}

\end{document}

%% file: Sections/abstract.tex
\begin{abstract}
Multimodal Large Language Models (MLLMs) still struggle with fine-grained visual counting, especially as the number of objects increases. We study where these errors come from and identify two recurring bottlenecks. The first arises during object individuation: non-overlapping patchification can split a single object across multiple visual tokens, making it harder to form coherent object-level representations. The second appears during aggregation, where attention normalization can reduce the separation between representations of neighboring counts as numerosity grows. Based on these observations, we introduce ConvStack, a lightweight module that mixes local spatial information directly in the visual token space and injects the resulting features through zero-initialized residual connections. This design helps preserve local object structure while providing a more stable visual signal for downstream counting. Although ConvStack is trained only on counting data while updating about 0.5\% of the model parameters, its gains extend beyond counting to broader spatial understanding, while general visual performance remains largely intact. Our mechanistic analysis shows that ConvStack improves the separation of counting representations and the encoding of spatial information within MLLMs.
\end{abstract}

%% file: Sections/intro.tex
\section{Introduction}

\begin{figure}[!ht]
    \centering
    \includegraphics[width=\linewidth]{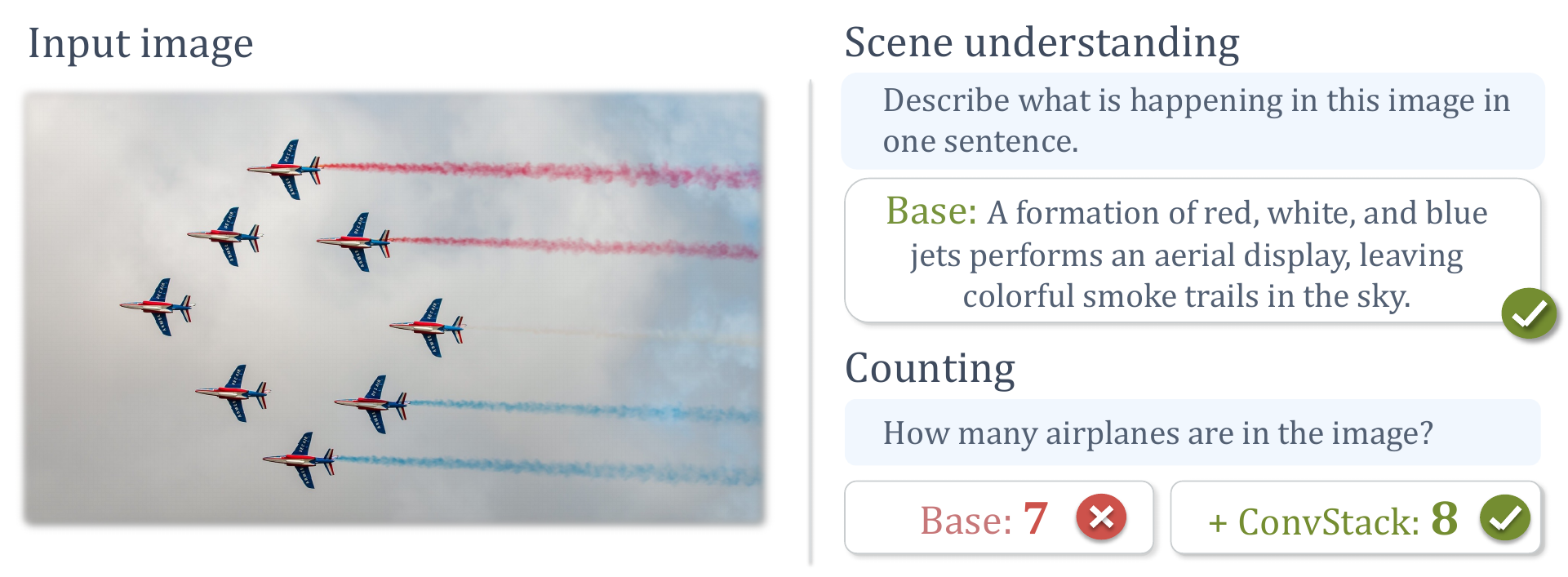}
    \caption{\textbf{Scene understanding versus counting.}
    Qwen3-VL-8B correctly describes the aerial display but counts seven airplanes instead of eight.
    \methodname{} (ours) recovers the correct count.}
    \label{fig:teaser}
\end{figure}
\FloatBarrier

Modern Multimodal Large Language Models (MLLMs)~\citep{bai2025qwen3vltechnicalreport,wang2025internvl3,gemma2024deepmind} inherit much of their visual capability from pretrained Vision Transformer (ViT) encoders \citep{dosovitskiy2021image,radford2021learning,liu2023visual}. While these encoders provide strong global representations that align well with language models, they remain surprisingly fragile on tasks requiring precise visual grounding. As shown in \Cref{fig:teaser}, an MLLM may accurately describe the semantic contents of an image while simultaneously failing to correctly answer how many objects are present~\citep{hasani2025understanding,che2026countingcircuitsmechanisticinterpretability,guo2025your}.

We use visual counting as a minimal, highly revealing probe for this broader grounding deficit. Counting requires far more than merely detecting the presence of foreground semantics. A capable model must first parse raw patches into discrete, localized object units (\textit{object individuation}) and then accumulate those units into a discrete numerosity (\textit{counting aggregation}). In this work, we conduct a mechanistic analysis that identifies severe bottlenecks at both of these critical stages (\Cref{sec:method}). First, we reveal an \textit{individuation bottleneck} stemming from image patchification. Because standard Vision Transformers project nonoverlapping patches independently, they severely struggle to group fragmented geometric features across patch boundaries, making it difficult to form coherent representations of individual objects. Second, we identify an \textit{aggregation bottleneck} driven by attention normalization. Even when objects are successfully individuated, they must still be aggregated by the network's attention mechanisms. Under a homogeneous single-head model with fixed object and background representations, attention that favors objects yields progressively smaller signal increments between adjacent counts. This compression reduces the hidden state distances between neighboring counts and increases the sensitivity required for exact readout.

These twin diagnostic insights naturally dictate our architectural solution: a robust counting model requires a local spatial prior to heal fragmented object boundaries, coupled with an additive signaling pathway that bypasses attention compression. Driven by this intuition, we propose \methodname{}, a lightweight module that pairs localized convolutional feature extraction with deep residual fusion. By applying depthwise convolutions directly within the visual token space, \methodname{} forces neighboring tokens to communicate, effectively bridging patch boundaries to construct clear, individuated object representations. To reinforce this visual evidence downstream, we inject the adapted features into the MLLM through residual connections. Zero-initialized adapter output projections preserve the input features at initialization. Our analysis gives sufficient conditions under which additive visual evidence supports stable count readout at the injection point.

Through experiments (\Cref{sec:experiments}), we evaluate how \methodname{} affects counting, spatial understanding, and general visual performance. Remarkably, by fine tuning \methodname{} \textit{exclusively} on counting datasets, the model acquires substantial improvements not only in dense object counting but also across a wide range of broader spatial understanding benchmarks. The vision and language backbones remain frozen while the adapters and visual merger are trained. General visual performance remains broadly comparable, with gains on some benchmarks and decreases on others.

In summary, this paper makes three primary contributions:
\begin{itemize}[leftmargin=2em]
    \item \textbf{Mechanistic Diagnosis:} Through theoretical analysis and synthetic diagnostics, we identify and formalize object individuation failure and counting aggregation collapse as two fundamental bottlenecks that systematically hinder the numeric and spatial reasoning capabilities of MLLMs.
    \item \textbf{Architectural Solution:} We introduce \methodname{}, a lightweight convolutional module that operates directly in the visual token space to explicitly aggregate and inject local geometric structures via residual connections.
    \item \textbf{Empirical Validation:} We demonstrate that training \methodname{} exclusively on counting tasks significantly boosts both dense counting in real images and broader spatial understanding in modern MLLMs (e.g., Qwen3-VL-8B), with broadly comparable general visual performance.
\end{itemize}

%% file: Sections/related_work.tex
\section{Related Work}

\noindent\textbf{Vision Transformers and convolutional bias.}
ViTs tokenize an image into nonoverlapping patches and rely on attention to mix spatial information \citep{dosovitskiy2021image}. This design provides global receptive fields but weakens early local inductive bias compared with CNNs \citep{lecun1998gradient}. Hybrid architectures reintroduce convolutional structure through convolutional attention projections, combinations of convolution and attention, and soft convolutional priors \citep{wu2021cvt,dai2021coatnet,dascoli2021convit}. Early convolutions improve ViT optimization and visual recognition \citep{xiao2021early}, while Convpass demonstrates the value of convolutional inductive bias for adaptation with few trainable parameters \citep{jie2022convolutional}. \methodname{} shares this emphasis on local spatial processing and convolutional adaptation. We examine how these ideas support object individuation and count readout, combining local visual features with residual routes and drawing on DeepStack's multilayer visual injection \citep{meng2024deepstack}.

\noindent\textbf{Visual counting.}
Counting has long been treated as a problem of localizing countable evidence and aggregating it into a cardinality estimate. Classical counting methods and methods for crowd counting often estimate density maps or evidence about individual objects before summation \citep{lempitsky2010learning}, while counting everyday objects emphasizes localization of individual instances and the difficulty of natural scenes \citep{chattopadhyay2017counting}. Human enumeration also distinguishes small exact counts from larger approximate number judgments \citep{kaufman1949discrimination,trick1994why,feigenson2004core}, motivating examination of how counting performance varies with numerosity. We use synthetic images to isolate whether an encoder forms object units under controlled changes in patch alignment, shape, color, and size, and examine performance on small versus large counts under balanced count frequencies. These experiments probe object individuation and count readout without assuming that a model implements human cognition or exhibits a universal subitizing boundary.

\noindent\textbf{Visual grounding in language models and shortcut diagnostics.}
Contrastive models such as CLIP \citep{radford2021learning} and MLLMs such as LLaVA \citep{liu2023visual} rely on visual encoders whose token representations are aligned to language. Spatial reasoning benchmarks and diagnostic datasets such as CLEVR \citep{johnson2017clevr} show that high overall VQA accuracy can coexist with shortcut behavior and weak compositional grounding. Our controlled results make the shortcut issue concrete: a model can appear to count perfectly when the data exposes a cue from the patch grid or object area. The proposed analysis therefore treats counting accuracy as meaningful only when the construction rules out such shortcuts and requires local object individuation.

%% file: Sections/method.tex
\section{Method}
\label{sec:method}

We view visual counting as two complementary computations: \emph{object individuation}, which forms distinct object representations, and \emph{information aggregation}, which combines the resulting evidence to determine their number.
We examine how local spatial features support object formation and how attention aggregates object evidence into a representation sensitive to count.

\subsection{Patchification Makes Object Individuation Harder}
\label{sec:patch-individuation}

\noindent\textbf{ViTs learn large counts more slowly.}
We define visual counting as predicting $n=|\mathcal O(x,q)|$, where $\mathcal O(x,q)$ denotes the set of objects in image $x$ specified by query $q$.
To study counting under controlled conditions, we use \syndot{} and \synpoly{}, introduced by \citet{che2026countingcircuitsmechanisticinterpretability}.
These datasets contain black disks and colored polygons, respectively, on white backgrounds, with $1$ to $10$ objects per image in our experiments.

\begin{figure}[!htbp]
    \centering
    \includegraphics[width=0.88\linewidth]{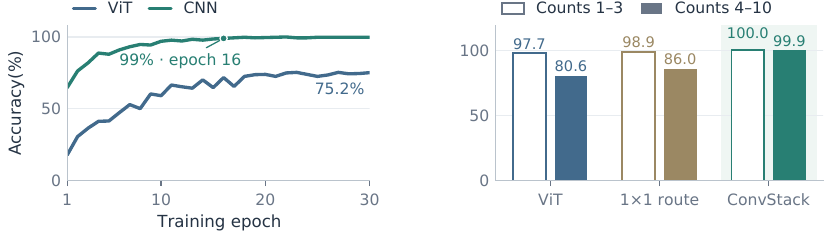}
    \caption{\textbf{Local spatial features improve counting.}
    Left: CNN learns faster than ViT on \synpoly{}.
    Right: \methodname{} nearly closes the accuracy gap between small and large counts. The $1\times1$ route control leaves a substantial gap.}
    \label{fig:method-locality}
\end{figure}

We begin with a simple count classification task on \synpoly{} to examine how the visual encoder affects counting.
We compare a standard ViT and a CNN, each with four layers, training both on \synpoly{} with balanced count frequencies and evaluating their frozen output embeddings at each epoch using a linear classification head.
Under the same training protocol, the CNN reaches $99\%$ probe accuracy at epoch $16$, whereas the ViT reaches only $75.2\%$ by epoch $30$ (\Cref{fig:method-locality}, left).
The ViT achieves much lower accuracy for large counts than for small counts.
In a separate three-seed comparison with balanced training counts, its mean accuracy is $97.7\%$ for counts $1$ to $3$ but $80.6\%$ for counts $4$ to $10$.

\noindent\textbf{Moving the same disk changes the learning problem.}
We suspect that this difference stems from the ViT's patchification of the input image.
We next use \syndot{} to isolate patch geometry.
A ViT divides an image $x\in\mathbb R^{H\times W\times3}$ into $M=HW/p^2$ nonoverlapping patches of size $p\times p$.
Let $P_i$ denote the image region of patch $i$ and $S_k$ the support of target object $k$.
The fraction of that object assigned to each patch is
\begin{equation}
    \rho_{ki}=\frac{|S_k\cap P_i|}{|S_k|},\qquad
    \sum_{i=1}^{M}\rho_{ki}=1.
    \label{eq:patch-fragments}
\end{equation}
We fix the disk radius to $7$ pixels and patch width to $16$ pixels, and vary placement relative to the grid (\Cref{fig:method-patch-geometry}).
At a \emph{patch center}, one token contains the entire disk: $\max_i\rho_{ki}=1$.
At a \emph{shared edge}, two tokens contain equal halves: $\rho_{ki}=\rho_{kj}=1/2$.
With \emph{random placement}, splits are generally asymmetric. The chosen sizes of disks and patches often leave a dominant fragment in one token.

The displayed center, random, and edge runs first reach $100\%$ probe accuracy in every class at epochs $1$, $9$, and $14$, respectively, and all finish at $100\%$.
When each disk lies within a distinct patch, the object count equals the number of occupied patches.
Splitting disks across patch boundaries breaks this direct correspondence and distributes object evidence across multiple tokens.

\begin{figure}[!htbp]
    \centering
    \includegraphics[width=0.84\linewidth]{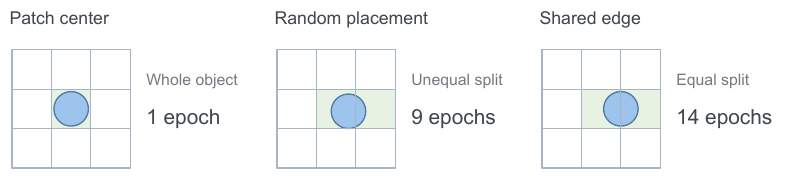}
    \caption{\textbf{Patch geometry changes learning speed.}
    Numbers indicate the first epoch with $100\%$ accuracy across all count classes.}
    \label{fig:method-patch-geometry}
\end{figure}
\FloatBarrier

\noindent\textbf{Overlapping filters provide a local prior.}
These observations motivate introducing a local spatial prior into the visual encoder.
Standard patch embedding projects each nonoverlapping patch independently, without combining information across patch boundaries.
Shared, overlapping convolutional filters aggregate neighboring pixels and provide a natural way to introduce this prior \citep{xiao2021early,jie2022convolutional}.
For the controlled ViT experiments, we instantiate this idea with a convolutional branch operating on image pixels that extracts hierarchical local features and injects them into the ViT through residual connections (\Cref{app:object-formation}).

To examine the role of learned spatial filtering, we compare \methodname{} with a control that replaces each $3\times3$ convolution with a $1\times1$ convolution, equivalent to a shared linear layer applied independently at each spatial location.
This control retains the same depth, channel widths, pooling operations, and injection points.
On \synpoly{}, \methodname{} achieves $99.9\%$ mean accuracy across three seeds and reduces the accuracy gap between small and large counts from $17.0$ to $0.1$ percentage points.
The $1\times1$ route control retains a gap of $12.9$ percentage points (\Cref{fig:method-locality}, right).
These results support learned spatial filtering as a useful inductive bias for visual counting.

\begin{table}[!htbp]
    \caption{\textbf{Counting on \synpoly{}.} Results after $15$ epochs with $500$ training samples per class. Metric definitions and the evaluation protocol are given in \Cref{app:counting-metrics}.}
    \label{tab:synpoly-main}
    \begin{center}
        \begin{tabular}{lrrrr}
            \toprule
            Model & Accuracy & MAE & RMSE & Class sep. \\
            \midrule
            ViT & $0.779$ & $0.287$ & $0.408$ & $1.570$ \\
            CNN & $0.998$ & $0.027$ & $0.060$ & $2.995$ \\
            ViT+\methodname{} & $0.998$ & $0.009$ & $0.037$ & $5.835$ \\
            \bottomrule
        \end{tabular}
    \end{center}
\end{table}

\begin{figure}[!htbp]
    \centering
    \includegraphics[width=0.78\linewidth]{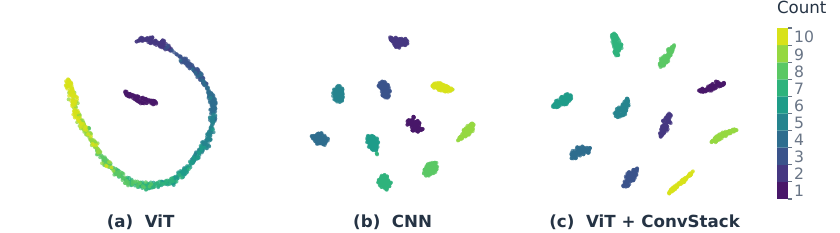}
    \caption{\textbf{ConvStack improves the separation of count classes.}
    t-SNE visualizations of \synpoly{} image embeddings, colored by object count.
    \methodname{} forms compact, distinct clusters and reduces the class overlap seen in ViT, especially at larger counts.}
    \label{fig:method-count-geometry}
\end{figure}
\FloatBarrier

\noindent\textbf{ConvStack improves the structure of count classes and counting performance.}
We next examine how this spatial prior affects count representations by comparing ViT, CNN, and ViT+\methodname{} trained on \synpoly{} for $15$ epochs with $500$ images per count class.
The t-SNE \citep{vandermaaten2008visualizing} visualizations of their final image embeddings show substantial overlap among the ViT's count classes, especially at larger counts (\Cref{fig:method-count-geometry}).
Adding \methodname{} produces more compact clusters of images with the same count and clearer separation between count classes.
Consistent with this visual pattern, the class separation score measured on the image embeddings increases from $1.570$ to $5.835$, approximately $3.7\times$ the ViT score and above the CNN's $2.995$ (\Cref{tab:synpoly-main}).
Counting accuracy also rises from $77.9\%$ to $99.8\%$, a gain of $21.9$ percentage points, matching the CNN baseline.
In the separate comparison across three seeds (\Cref{fig:method-locality}, right), \methodname{} reduces the accuracy gap between small and large counts to $0.1$ percentage points, versus $12.9$ percentage points for the $1\times1$ route control.
These results show that local spatial filtering strengthens visual count representations, particularly at larger counts.

\subsection{Count Separation Through Visual Residuals}
\label{sec:count-readout}

\begin{figure}[!htbp]
    \centering
    \includegraphics[width=\linewidth]{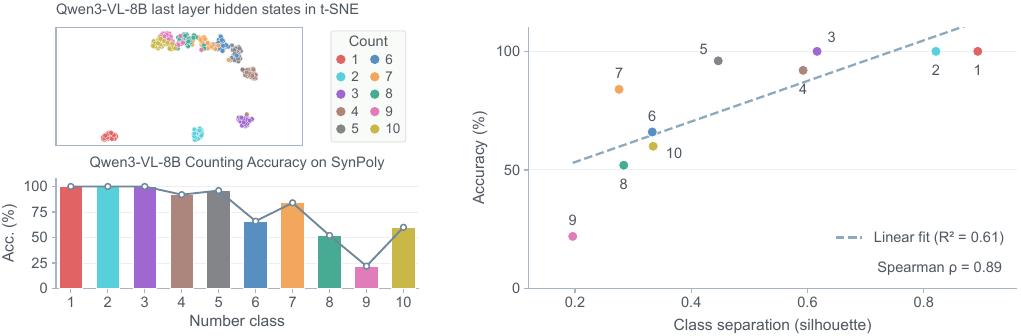}
    \caption{\textbf{Separation of count classes tracks counting accuracy.}
    Qwen3-VL-8B on \synpoly{}: t-SNE of final layer representations and accuracy for each class (left), and accuracy versus hidden state silhouette (right).}
    \label{fig:count-separation-motivation}
\end{figure}

We further examine the LLM backbone of Qwen3-VL-8B and observe a positive association between mean hidden state silhouette and counting accuracy across the ten count classes ($r=0.78$, $R^2=0.61$, \Cref{fig:count-separation-motivation}).
We compare these representations with \methodname{} using silhouette and adjacent class separation in \Cref{app:count-geometry}.
This observation motivates us to investigate how attention aggregation~\citep{che2026countingcircuitsmechanisticinterpretability} preserves distinctions between counts.

\noindent\textbf{The contribution of an additional object.}
To isolate aggregation from individuation, consider $n$ objects each contained in a distinct patch: $\rho_{ki}\in\{0,1\}$ for all $k=1,\ldots,n$ and $i=1,\ldots,M$, with no patch containing more than one object.
Assuming one visual token per patch at the LLM input, the number of object tokens is then $N=n$.
Consider a single attention head over $M$ image tokens, with $n$ homogeneous object tokens and fixed object and background representations (\Cref{app:attention-proof}).
Writing $r=\exp(s_o-s_b)$ for the attention weight of an object token relative to a background token, the total attention mass on object tokens is $A(n)=nr/[M+n(r-1)]$.
The attention increment from one additional object is
\begin{equation}
    \Delta A(n)=A(n+1)-A(n)
       =\frac{rM}{[M+n(r-1)][M+(n+1)(r-1)]},\quad 0\leq n<M.
    \label{eq:count-margin}
\end{equation}
For attention that favors objects ($r>1$), $\Delta A(n)$ decreases as the number of objects grows.
With a fixed incoming residual $h_{\mathrm{in}}$ and output projection $W_O$, the hidden state after attention is
\begin{equation}
    h_n=h_{\mathrm{in}}+W_O\bigl[v_b+A(n)(v_o-v_b)\bigr].
    \label{eq:attention-hidden-state}
\end{equation}
The difference between adjacent counts is therefore $h_{n+1}-h_n=\Delta A(n)W_O(v_o-v_b)$, giving
\begin{equation}
    D_n=\|h_{n+1}-h_n\|_2=K\Delta A(n),
    \qquad K=\|W_O(v_o-v_b)\|_2.
    \label{eq:hidden-count-margin}
\end{equation}
Thus, for $K>0$, shrinking attention increments translate directly into smaller hidden state distances between adjacent counts, even with perfectly individuated objects.
This compression helps explain why hidden state clusters for larger counts increasingly overlap and become harder to distinguish in \Cref{fig:count-separation-motivation}.
Recovering exact counts from these closely spaced states requires a scalar decoder with sensitivity at least $1/D_n$ (\Cref{app:attention-proof,app:robustness-proof}), making readout increasingly sensitive to hidden state perturbations under this model.



To provide a more reliable source of visual evidence, we adopt a residual injection strategy~\citep{meng2024deepstack}. We augment the LLM hidden state $h_n$ after attention with a projected visual feature $r_n$ for count $n$, scaled by a fixed gate $g$, yielding the fused state $\widetilde h_n$. We formalize ``additive evidence'' by assuming that the visual route contributes a signal along a shared unit direction that grows linearly with count, with slope $c>2\eta\geq0$ and error bounded by $\eta$. We also assume that the base stream's variation between adjacent counts along this direction is bounded by $\beta$.
Under these conditions, the fused state $\widetilde h_n$ and its effective separation margin $\kappa$ become:
\begin{equation}
    \widetilde h_n=h_n+gr_n,\qquad
    \kappa=|g|(c-2\eta)-\beta,\qquad \text{for } g\neq0.
    \label{eq:evidence-residual}
\end{equation}

\begin{theorem}[Stable count readout from preserved residual evidence]
\label{thm:residual-count-readout}
Assuming the evidence and interference bounds outlined above hold uniformly up to some maximum count $n_{\max}$ (\Cref{ass:residual-evidence}), any configuration where $\kappa>0$ ensures the residual states are separated by a margin completely independent of the count itself. Under these conditions, there exists an exact scalar decoder $F:\mathbb R^d\to\mathbb R$ with $F(\widetilde h_n)=n$ such that:
\begin{equation}
    \|\widetilde h_m-\widetilde h_n\|_2\geq\kappa|m-n|,\qquad
    \operatorname{Lip}(F)\leq1/\kappa.
    \label{eq:residual-margin}
\end{equation}
By rounding $F$ to the nearest integer, one can exactly recover the true count even when the states are corrupted by any additive noise with $\ell_2$ norm strictly less than $\kappa/2$.
\end{theorem}

When the additive visual signal outweighs the base stream's interference ($\kappa>0$), the theorem guarantees a count separation margin and the existence of a decoder with sensitivity at most $1/\kappa$ throughout the stated count range.

This is a sufficient condition at the injection point. Real images need not satisfy the evidence assumptions, and the theorem does not guarantee that training learns these features or that subsequent language layers preserve them (\Cref{app:residual-readout-proof}). It motivates extracting local visual cues and reinforcing them at multiple network depths.

\subsection{The ConvStack Architecture}
\label{sec:convstack}

Driven by the need to continuously inject stable counting evidence, we propose \methodname{}, an architecture that pairs localized convolutional feature extraction with a deep residual fusion mechanism. To enhance the spatial reasoning capabilities of pretrained MLLMs, \methodname{} operates directly in the visual token space (\Cref{fig:convstack-overview}).

\begin{figure}[!htbp]
    \centering
    \includegraphics[width=0.84\linewidth]{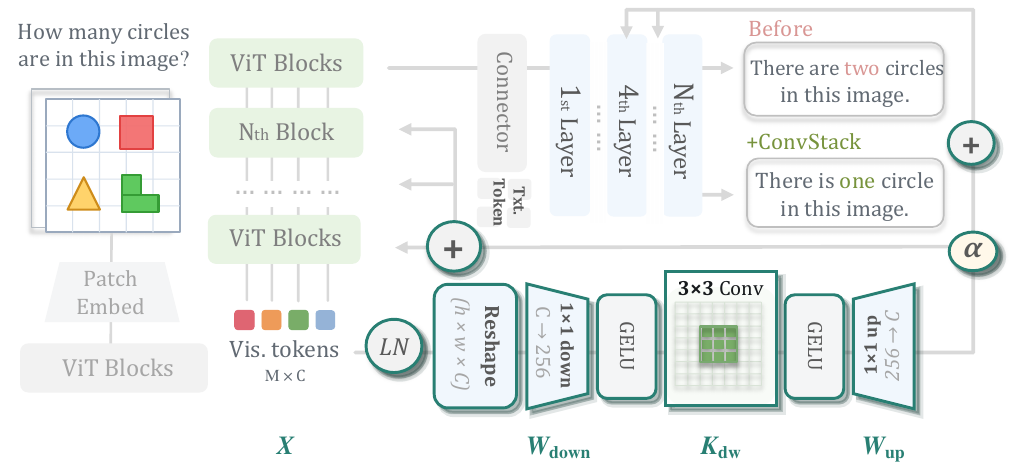}
    \caption{\textbf{Overview of \methodname{}.} Teal outlines highlight the \methodname{} components. Variable names correspond to \Cref{eq:token-convstack}.}
    \label{fig:convstack-overview}
\end{figure}

At each adapter location, we reshape the visual tokens into a 2D grid $X$ and update it with a convolutional bottleneck:
\begin{equation}
    \mathcal C_\theta(X)=W_{\mathrm{up}}*
       \operatorname{GELU}\left(K_{\mathrm{dw}}*
       \operatorname{GELU}(W_{\mathrm{down}}*\operatorname{LN}(X))\right),
       \qquad \widetilde X=X+\alpha\mathcal C_\theta(X).
    \label{eq:token-convstack}
\end{equation}
Here, $W_{\mathrm{down}}$ and $W_{\mathrm{up}}$ are $1\times1$ projections that compress and restore the channel dimension, while the $3\times3$ depthwise convolution $K_{\mathrm{dw}}$ mixes neighboring tokens in the compressed space. A learnable scalar $\alpha$ scales the correction. The updated features $\widetilde X$ are flattened into visual tokens for fusion with the hidden states at selected LLM layers. We initialize $W_{\mathrm{up}}$ to zero so that each adapter starts as an identity mapping, preserving the pretrained features at initialization.

%% file: Sections/exp.tex
\section{Experiments}
\label{sec:experiments}

\subsection{Experimental Setup}
\label{sec:experimental-setup}

\noindent\textbf{Data and baselines.}
We implement \methodname{} on Qwen3-VL-8B \citep{bai2025qwen3vltechnicalreport} and compare it with Spatial-MLLM-3B \citep{wu2025spatialmllmboostingmllmcapabilities}, LLaVA-SP Cropping-7B \citep{Lou_2025_ICCV}, and Honeybee-7B \citep{cha2023honeybee}.
We assess counting in real images on PixMo-Count \citep{deitke2025molmo} and CountBenchQA \citep{beyer2024paligemma}. We evaluate spatial understanding on CV-Bench \citep{tong2024cambrian}, SAT-real \citep{ray2024sat}, SpatialEval \citep{wang2024picture}, and SAT-Spatial \citep{batra2025satspatialvqa}. We assess general visual capability on POPE \citep{li2023evaluating}, MME \citep{fu2026mme}, MMBench-EN \citep{liu2024mmbench}, MMMU \citep{yue2024mmmu}, RealWorldQA \citep{xai2024realworldqa}, and MathVista \citep{lu2024mathvista}.
Evaluation uses deterministic decoding and each model's image processor.
Dataset descriptions, evaluation splits, and training variants are detailed in \Cref{sec:datasets}.

\noindent\textbf{Training.}
Our standard training mixture combines examples from the PixMo-Count training split with $2{,}000$ \synpoly{} and $2{,}000$ \syndot{} examples, following the synthetic datasets introduced by \citet{che2026countingcircuitsmechanisticinterpretability}.
For Qwen3-VL, we train convolutional adapters at vision layer $18$ (ablation in \Cref{app:vit-injection}) and all native DeepStack levels, together with the visual merger.
The remaining pretrained parameters are frozen.
We train for one epoch with an effective batch size of $32$, a learning rate of $5\times10^{-5}$, and loss on tokens in counting answers.

\subsection{Evaluation}
\label{sec:evaluation}

\noindent\textbf{Counting and Spatial Understanding.}
As shown in \Cref{tab:counting-spatial}, \methodname{}-8B consistently outperforms both the base Qwen3-VL-8B and alternative spatial adapters. The most striking gains are in dense counting, where it achieves $75.19\%$ on PixMo-Count, an absolute improvement of $9.73$ percentage points over the base model ($65.46\%$). Furthermore, the injected local features systematically benefit broader spatial reasoning, establishing top scores on benchmarks like CV-Bench Spatial ($93.20\%$) and SAT-Spatial ($76.81\%$). Notably, \methodname{} proves superior to existing workarounds like Pool SFE or Crop+DFI~\citep{Lou_2025_ICCV}, and avoids the severe performance collapse caused by aggressive strategies for token reduction, such as C-Abstractor~\citep{cha2023honeybee}.

\begin{table}[htbp]
\centering
\small
\setlength{\tabcolsep}{3pt}
\caption{\textbf{Evaluation on counting and spatial understanding benchmarks.} CV-Bench Spatial averages Relation, Depth, and
Distance accuracy. The best result is
\textbf{bolded}, and the second best is
\underline{underlined}. 
Dashes denote results that are unavailable or incompatible with the scorer.}
\label{tab:counting-spatial}

\resizebox{\linewidth}{!}{%
\begin{tabular}{@{}lrrrrrr@{}}
\toprule
Method
& PixMo-Count
& CountBenchQA
& \shortstack{CV-Bench\\Spatial}
& SAT-real
& SpatialEval
& SAT-Spatial \\
\midrule

\rowcolor{groupbg}
\multicolumn{7}{@{}c@{}}{%
  \textcolor{grouptext}{\bfseries\itshape Models Specialized for Spatial Tasks}%
} \\

Spatial-MLLM-3B
& 52.67 & 65.78 & 75.22 & 60.00 & 50.00 & 58.76 \\

LLaVA-SP Cropping-7B
& 42.94 & 45.62 & 63.59 & 52.00 & 31.33 & 60.61 \\

Honeybee-7B (C-Abstractor)
& 37.98 & 54.79 & 60.50 & 51.00 & 28.00 & 57.63 \\

\midrule

\rowcolor{groupbg}
\multicolumn{7}{@{}c@{}}{%
  \textcolor{grouptext}{\bfseries\itshape Qwen3-VL-8B Scale}%
} \\

Qwen3-VL-8B
& 65.46 & \underline{89.82} & 92.34 & 59.33 & 61.33 & \underline{76.79} \\

Qwen3-VL-8B + Crop+DFI
& 71.43 & 85.74 & \underline{93.04} & \underline{62.13}
& \underline{65.22} & -- \\

Qwen3-VL-8B + Pool SFE
& \underline{73.52} & 89.21 & 92.40 & 58.67 & 64.64 & 75.05 \\

Qwen3-VL-8B + C-Abstractor
& 36.95 & 29.53 & -- & 29.33 & 30.46 & -- \\

\methodname{}-8B
& \textbf{75.19} & \textbf{90.22} & \textbf{93.20} & \textbf{63.33}
& \textbf{66.67} & \textbf{76.81} \\

\bottomrule
\end{tabular}%
}
\end{table}

\noindent\textbf{General Visual Capability.}
As shown in \Cref{tab:general-visual}, \methodname{}-8B maintains broadly comparable general visual performance to Qwen3-VL-8B, with improvements on four of the six benchmarks. POPE F1 rises from $87.83\%$ to $89.52\%$, MME from $2219.6$ to $2251.0$, and RealWorldQA from $69.93\%$ to $70.33\%$, while MMBench and MathVista show small decreases.

\begin{table}[htbp]
    \centering
    \small
    \setlength{\tabcolsep}{3pt}
    \caption{\textbf{General visual capability.} POPE reports F1. Scores are percentages except MME. The better result in each column is bolded.}
    \label{tab:general-visual}
    \begin{tabular*}{\linewidth}{@{\extracolsep{\fill}}lrrrrrr@{}}
        \toprule
        Method & POPE & MME & MMBench & MMMU & RealWorldQA & MathVista \\
        \midrule
        Qwen3-VL-8B & 87.83 & 2219.6 & \textbf{88.67} & 59.00 & 69.93 & \textbf{66.60} \\
        \methodname{}-8B & \textbf{89.52} & \textbf{2251.0} & 88.60 & \textbf{59.89} & \textbf{70.33} & 65.00 \\
        \bottomrule
    \end{tabular*}
\end{table}

\FloatBarrier
\subsection{Ablation Study}
\label{sec:pretrained-ablation}

\Cref{tab:pretrained-ablation} compares \methodname{} against three alternative fine tuning strategies on Qwen3-VL-8B. All variants share the exact same training data and hyperparameters. \emph{ViT-stack only} trains the ConvStack with ViT layer residual connections only. \emph{LLM-stack only} directly routes the ConvStack head output to the residual connection of the LLM backbone. \emph{$1\times1$ convolution kernel} replaces the $3\times3$ depthwise convolution with a $1\times1$ projection, removing spatial mixing within the adapter. \emph{LoRA} applies LoRA throughout the 8B MLLM. Our method trains $42.8$M parameters in its residual adapters and visual merger. The results confirm that explicit local token mixing ($3\times3$ convolution) is critical.

\begin{table}[!htbp]
    \centering
    \small
    \setlength{\tabcolsep}{12pt}
    \caption{Ablation study (accuracy \%).}
    \label{tab:pretrained-ablation}
    \begin{tabular}{@{}lrr@{}}
        \toprule
        Configuration & PixMo-Count & SAT-real \\
        \midrule
        \textbf{\methodname{}-8B} & \textbf{75.19} & \textbf{63.33} \\
        \quad ViT-stack only & 73.09 & 60.67 \\
        \quad LLM-stack only & 74.05 & 61.33 \\
        \quad $1\times1$ convolution kernel & 73.09 & 59.33 \\
        \quad LoRA & 73.85 & 62.00 \\
        \bottomrule
    \end{tabular}
\end{table}
\FloatBarrier

\subsection{Mechanistic Analysis}
\label{sec:spatial-analysis}

\noindent\textbf{Internal representations of counting.}
To understand how \methodname{} improves counting, we analyze its internal representations using $500$ \synpoly{} images with object counts ranging from $1$ to $10$.
As shown in \Cref{fig:synpoly-final-tsne}, \methodname{} produces clearer clusters for each count, alongside an accuracy increase from $56.8\%$ to $98.8\%$ for counts $6$ to $10$.
In the original hidden space, mean cosine silhouette rises from $0.479$ to $0.707$, and separation improves for every adjacent count pair (\Cref{app:count-geometry}).

\begin{figure}[!ht]
    \centering
    \includegraphics[width=\linewidth]{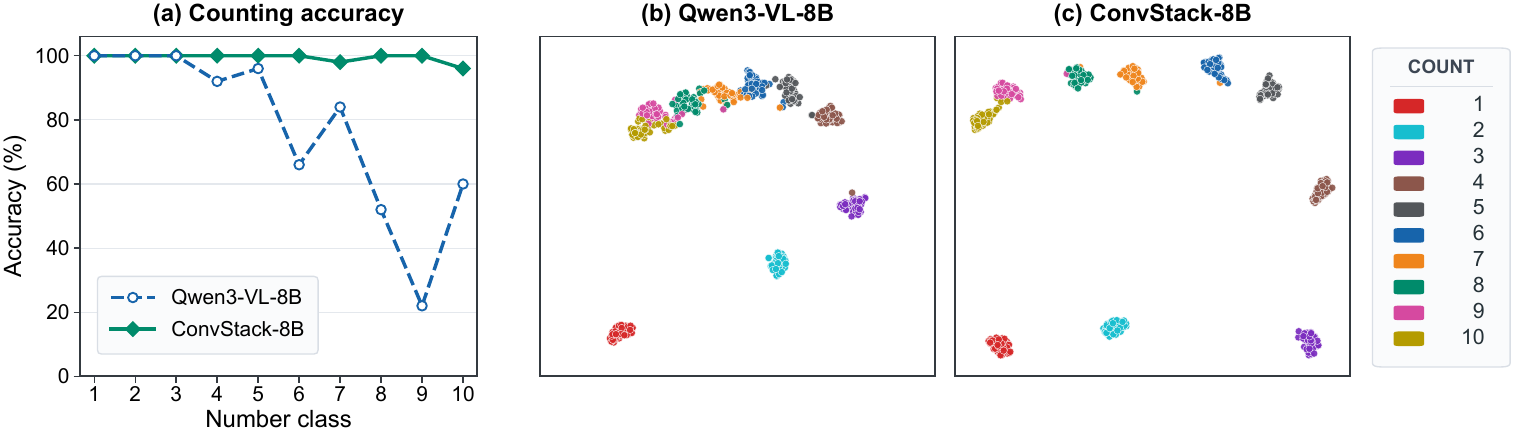}
    \caption{\textbf{Counting accuracy and representations on \synpoly{}.}
    (a) Accuracy by number class, with $50$ images per class.
    (b,c) t-SNE of the final layer representation of the last prompt token for counts $1$ to $10$.}
    \label{fig:synpoly-final-tsne}
\end{figure}

\noindent\textbf{Probing spatial features via S-Space.}
To verify that these representation benefits extend beyond counting to general spatial grounding, we apply S-Space~\citep{mirros2026sspace} to probe intermediate model representations on $500$ MSCOCO images. S-Space allows us to linearly read out object coordinates (horizontal, vertical, and depth) directly from the hidden states.

\Cref{fig:sspace-analysis}(a) highlights a representative failure of the base model, which incorrectly judges the dining table to be farther away than the refrigerator. \methodname{} corrects this error, and looking at the internal S-Space readout (\Cref{fig:sspace-analysis}(b)), we see the corresponding depth probe shift dramatically from $-0.28$ (incorrectly farther) to $+0.08$ (correctly closer). Zooming out to the full set of $500$ images (\Cref{fig:sspace-analysis}(c)), \methodname{} yields consistent aggregate gains in representational accuracy across all three dimensions ($+3.5\%$ horizontal, $+2.3\%$ vertical, and $+2.1\%$ depth). This confirms that explicitly mixing local visual features forces the network to maintain a much stronger internal grounding of physical object relations.

\begin{figure}[!htbp]
    \centering
    \includegraphics[width=\linewidth]{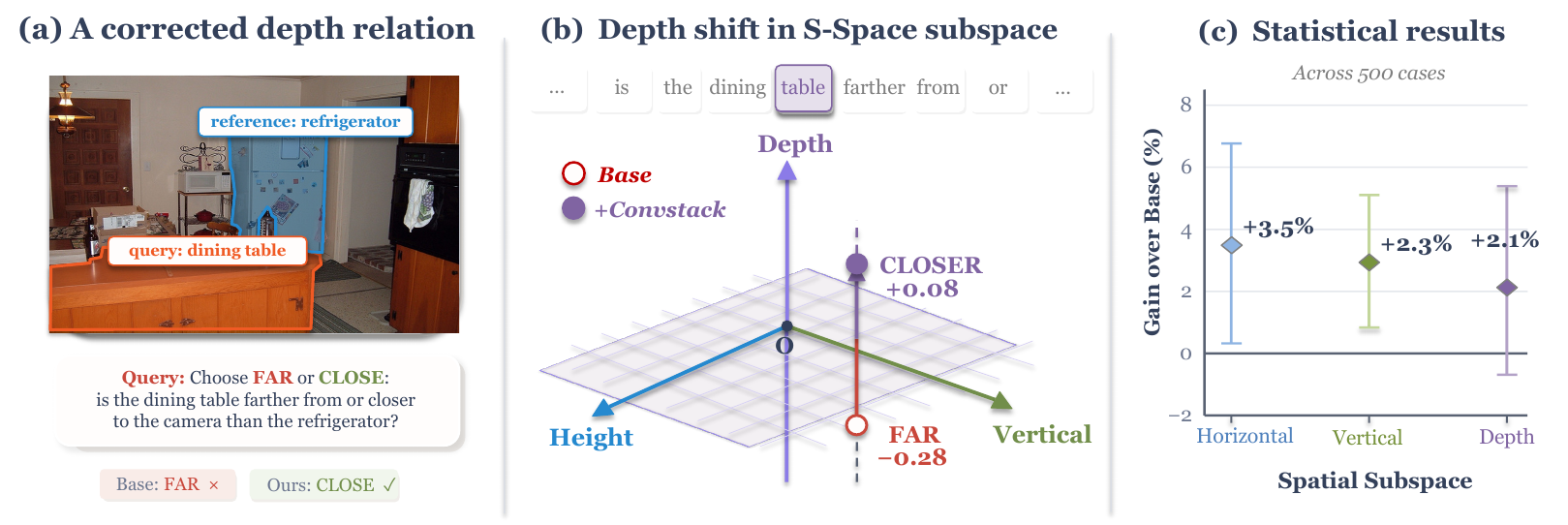}
    \caption{\textbf{Spatial understanding analyzed with S-Space on $500$ MSCOCO images.}
    (a) \methodname{} corrects the base model's depth judgment for a dining table relative to a refrigerator.
    (b) The corresponding S-Space depth readout shifts from $-0.28$ (farther) to $+0.08$ (closer).
    (c) Aggregate gains over the base model are positive for horizontal, vertical, and depth evaluations.}
    \label{fig:sspace-analysis}
\end{figure}
\FloatBarrier

%% file: Sections/conclusion.tex
\section{Conclusion}

In this work, we identified two critical bottlenecks in the object individuation and counting aggregation processes that cause MLLMs to struggle with visual counting, especially for large numbers. To address these limitations, we proposed \methodname{}, an architecture that effectively improves the separation of internal counting representations. Remarkably, by training exclusively on counting tasks, the model achieves significant improvements in both visual counting and broader spatial understanding. These gains are accompanied by broadly comparable general visual performance, supporting \methodname{} as a practical approach to strengthening numeric and spatial reasoning in multimodal foundation models.

%% file: Sections/appendix_datasets.tex
\section{Datasets}
\label{sec:datasets}

We describe the datasets used for training, counting diagnostics, and evaluation. The standard adaptation mixture uses the PixMo-Count training split and 2,000 examples each from \synpoly{} and \syndot{}.

\subsection{Counting Data and Controlled Diagnostics}

\noindent\textbf{PixMo-Count.}
PixMo-Count \citep{deitke2025molmo} contains natural images annotated with a target object category and its count, with point annotations provided for training. Its validation and test splits, verified by human annotators, cover counts from 2 to 10. We use examples from the training split for counting adaptation and evaluate accuracy for exact counts on the available test images. The current Qwen3 evaluation contains 524 test examples.

\noindent\textbf{\synpoly{}.}
\synpoly{}, introduced by \citet{che2026countingcircuitsmechanisticinterpretability}, contains colored polygons on a white background with variation in shape, size, color, and position. It tests counting under appearance variation while keeping the scene simple. We use 2,000 examples in the standard adaptation mixture. The controlled encoder experiments use balanced counts from 1 to 10 and resolutions chosen for each experiment. The counting suite for pretrained models additionally uses 380 generated test images at $512\times512$ resolution, with 20 examples for each count from 1 to 19. The representation analysis in \Cref{fig:synpoly-final-tsne} uses a separate set of 500 images, with 50 examples per count from 1 to 10.

\noindent\textbf{\syndot{}.}
\syndot{}, also introduced by \citet{che2026countingcircuitsmechanisticinterpretability}, contains black disks on a white background. Its simple appearance makes it useful for isolating the relationship between object placement and patch geometry. We use 2,000 examples in the standard adaptation mixture. For controlled diagnostics, we construct variants with objects at patch centers, at patch edges, or at random positions, with counts from 1 to 10. The counting suite for pretrained models uses 380 generated $512\times512$ test images covering counts from 1 to 19, with 20 examples per count.

\noindent\textbf{CountBenchQA.}
CountBenchQA \citep{beyer2024paligemma} converts CountBench \citep{paiss2023teaching} into a visual question answering task by adding questions about object quantity. The original images contain between 2 and 10 instances of a target object. Our additional counting evaluations use the public test release of 491 images and accuracy for exact counts.

\subsection{Spatial Understanding Benchmarks}

\noindent\textbf{CV-Bench.}
CV-Bench \citep{tong2024cambrian} recasts annotated images from ADE20K, COCO, and Omni3D as visual questions about counting, spatial relations, depth order, and relative distance. Our CV-Bench Spatial score averages accuracy equally over Relation (650 questions), Depth (600), and Distance (600), covering 1,850 questions. The counting category is excluded from this spatial score.

\noindent\textbf{SAT-real.}
SAT-real is the test split of Spatial Aptitude Training (SAT) that uses real images \citep{ray2024sat} evaluating spatial relationships and changes involving objects or viewpoints, with one or two images per question.

\noindent\textbf{SpatialEval.}
SpatialEval \citep{wang2024picture} evaluates spatial relationships, position, counting, and navigation through tasks involving maps, grids, mazes, and real images. Our experiments use a fixed subset of 150 visual questions with multiple choices shared across models and report answer accuracy on this subset.

\noindent\textbf{SAT-Spatial.}
We use the public SAT-Spatial-VQA release \citep{batra2025satspatialvqa}, which contains 15,000 examples for spatial reasoning, each consisting of an image, a question, and an answer. Although the release names its split \texttt{train}, we use these examples only for evaluation. We retain the native questions and score normalized exact matches to the answer values.

\subsection{General Visual Benchmarks}

\noindent\textbf{POPE.}
POPE \citep{li2023evaluating} measures object hallucination through yes/no questions about whether a specified object appears in an image. We use the random, popular, and adversarial subsets based on COCO, each containing 3,000 questions, and report F1 over all 9,000 questions.

\noindent\textbf{MME.}
MME \citep{fu2026mme} evaluates perception and cognition across 14 tasks, including object recognition, spatial position, text recognition, and reasoning. Our evaluation contains 2,374 yes/no questions arranged in 1,187 pairs. Each task combines accuracy on individual questions with the accuracy of answering both questions in a pair correctly. We report the sum of perception and cognition scores.

\noindent\textbf{MMBench-EN.}
MMBench \citep{liu2024mmbench} is a benchmark with multiple choices per question covering a broad range of visual perception and reasoning abilities. We use the cached English development set with 4,377 evaluation records and report accuracy using deterministic option extraction.

\noindent\textbf{MMMU.}
MMMU \citep{yue2024mmmu} tests multimodal understanding and reasoning at an expert level across 30 academic subjects. We evaluate the validation split of 900 questions, which includes both questions with multiple choices and questions requiring free responses, and report accuracy with answer normalization for free responses.

\noindent\textbf{RealWorldQA.}
RealWorldQA \citep{xai2024realworldqa} evaluates understanding of everyday visual scenes, including photographs taken from vehicles. We use its test split of 765 examples, containing both questions with multiple choices and questions requiring short answers, and report accuracy after option extraction or answer normalization.

\noindent\textbf{MathVista.}
MathVista \citep{lu2024mathvista} evaluates mathematical reasoning grounded in visual inputs such as diagrams, charts, and natural images. We use the \texttt{testmini} split of 1,000 examples and report accuracy with local answer extraction followed by normalization of answer types and numerical precision.

\subsection{Images for Spatial Representation Analysis}

\noindent\textbf{MSCOCO.}
MSCOCO \citep{lin2014microsoft} provides natural images of everyday scenes with annotations of object instances. We select 500 images for the S-Space analysis \citep{mirros2026sspace} in \Cref{sec:spatial-analysis}, comparing horizontal, vertical, and depth information in the base model and its \methodname{} adaptation. These images form a set for representation analysis.

%% file: Sections/appendix_method.tex
\section{Method and Analysis}
\label{app:method}

\subsection{Attention Model and Compression of Count Margins}
\label{app:attention-proof}

\begin{assumption}[Homogeneous object and background tokens]
\label{ass:count-readout}
The head attends to a fixed number $M\geq2$ of visual tokens.
The LLM receives perfectly individuated objects: each of the $n\in\{0,\ldots,M\}$ objects occupies a distinct image token, so the number of object tokens is $N=n$.
The $n$ object tokens share logit $s_o$ and value $v_o$. The remaining $M-n$ tokens share logit $s_b$ and value $v_b$.
These logits and values, the output projection $W_O$, and the incoming residual $h_{\mathrm{in}}$ are independent of $n$.
Write $r=\exp(s_o-s_b)$.
\end{assumption}

\begin{proposition}[Compression of count margins under selective attention]
\label{prop:attention-compression}
Under \Cref{ass:count-readout}, the total attention mass on object tokens is $A(n)=nr/[M+n(r-1)]$, with the increment between adjacent counts given by \Cref{eq:count-margin}.
For a head that favors objects ($r>1$), $A(n)$ is strictly increasing and discretely strictly concave: $\Delta A(n+1)<\Delta A(n)$ for $0\leq n<M-1$.
Uniform attention ($r=1$) instead gives the constant interval $\Delta A(n)=1/M$.
\end{proposition}

\begin{proof}
Under \Cref{ass:count-readout}, the softmax denominator is
$ne^{s_o}+(M-n)e^{s_b}$.
Summing the weights of the $n$ object tokens and dividing by $e^{s_b}$ gives
$A(n)=nr/[M+n(r-1)]$.
Let $d_n=M+n(r-1)$, which is positive for $0\leq n\leq M$ and $r>0$.
Direct subtraction yields
\begin{equation}
    A(n+1)-A(n)
    =r\frac{(n+1)d_n-nd_{n+1}}{d_nd_{n+1}}
    =\frac{rM}{d_nd_{n+1}}>0.
\end{equation}
For $r>1$ and $0\leq n<M-1$,
\begin{equation}
    \Delta A(n+1)-\Delta A(n)
    =-\frac{2rM(r-1)}{d_nd_{n+1}d_{n+2}}<0.
\end{equation}
This proves strict discrete concavity.
If $r=1$, then $d_n=M$ and $\Delta A(n)=1/M$.
\end{proof}

The inverse quadratic approximation follows from the exact factorization, for $n\geq1$ and $r>1$,
\begin{equation}
    \Delta A(n)
    =\frac{rM}{(r-1)^2n(n+1)}
    \left(1+\frac{M}{n(r-1)}\right)^{-1}
    \left(1+\frac{M}{(n+1)(r-1)}\right)^{-1}.
\end{equation}
Both correction factors are close to one when $n(r-1)\gg M$.
This is an approximation within the finite count range $n<M$, rather than a limit that sends $n$ beyond the token budget.

\subsection{Decoder Sensitivity Under Compressed Count Intervals}
\label{app:robustness-proof}

Under \Cref{ass:count-readout}, $N=n$ and the hidden state is given by \Cref{eq:attention-hidden-state}.
Its distance between adjacent counts satisfies $D_n=K\Delta A(n)$ as stated in \Cref{eq:hidden-count-margin}, where $K=\|W_O(v_o-v_b)\|_2$.
If $D_n>0$, every scalar decoder exact on $h_n,h_{n+1}$ has Lipschitz constant at least $1/D_n$.
A decoder can distinguish these two counts under every additive perturbation of $\ell_2$ norm at most $\varepsilon$ if and only if $D_n>2\varepsilon$.
If $D_n\leq2\varepsilon$, every scalar decoder has absolute error at least $1/2$ for one of the two counts at some admissible perturbed state.
When $K=0$, the two unperturbed states coincide.

\begin{proof}
The total object and background weights are $A(n)$ and $1-A(n)$, so
$o_n=v_b+A(n)(v_o-v_b)$.
The fixed incoming residual cancels between adjacent states:
\begin{equation}
    h_{n+1}-h_n=\Delta A(n)W_O(v_o-v_b).
\end{equation}
Taking Euclidean norms proves \Cref{eq:hidden-count-margin}.
For any exact scalar decoder with finite Lipschitz constant $L_F$,
\begin{equation}
    1=|F(h_{n+1})-F(h_n)|
    \leq L_F\|h_{n+1}-h_n\|_2=L_FD_n.
    \label{eq:decoder-sensitivity}
\end{equation}
Hence $L_F\geq1/D_n$ whenever $D_n>0$. A decoder without a finite Lipschitz constant satisfies the bound trivially.

Define the full closed perturbation balls
\begin{equation}
    B_n(\varepsilon)=\{z:\|z-h_n\|_2\leq\varepsilon\}.
\end{equation}
If $D_n>2\varepsilon$, the balls around $h_n$ and $h_{n+1}$ are disjoint.
The decoder that selects the nearest center distinguishes them: for $z\in B_n(\varepsilon)$,
\begin{equation}
    \|z-h_{n+1}\|_2
    \geq D_n-\|z-h_n\|_2
    >\varepsilon
    \geq\|z-h_n\|_2,
\end{equation}
and symmetrically for $z\in B_{n+1}(\varepsilon)$.

Conversely, if $D_n\leq2\varepsilon$, the midpoint
$z_\star=(h_n+h_{n+1})/2$ lies in both balls.
A single decoder output cannot equal both labels there.
For every scalar decoder $F$, the triangle inequality gives
\begin{equation}
    1\leq |F(z_\star)-n|+|F(z_\star)-(n+1)|,
\end{equation}
so at least one of the two absolute errors is at least $1/2$.
Finally, $K=0$ implies $h_{n+1}=h_n$, making exact decoding impossible even at $\varepsilon=0$.
\end{proof}

The result is pairwise and conditional on \Cref{ass:count-readout}.
For a finite family of count states, a decoder that selects the nearest center is robust on all full perturbation balls if their minimum pairwise distance exceeds $2\varepsilon$.
In the homogeneous model with $K>0$, the states lie in order on a line, so the minimum pairwise distance is the smallest adjacent distance.
Heterogeneous values, additional attention heads, or residuals that depend on count can change this geometry and preserve other count cues.
Furthermore, actual image perturbations need not fill a norm ball: ball overlap proves an obstruction in the worst case for that uncertainty model, not the existence of a particular image transformation.

\subsection{\texorpdfstring{Proof of \Cref{thm:residual-count-readout}}{Stable residual count readout proof}}
\label{app:residual-readout-proof}

\begin{assumption}[Approximately additive residual evidence and bounded interference]
\label{ass:residual-evidence}
For a finite family of count states $n=0,\ldots,n_{\max}$, $n_{\max}\geq1$, let $h_n,r_n\in\mathbb R^d$ be the LLM hidden state before visual residual injection and the projected feature from the visual route, with fixed gate $g\neq0$.
There exist a unit vector $u$ and constants $a\in\mathbb R$, $c>2\eta\geq0$, $\beta\geq0$, shared across all counts, such that
\begin{equation}
    u^\top r_n=cn+a+\xi_n,\qquad |\xi_n|\leq\eta,
    \label{eq:additive-evidence}
\end{equation}
\begin{equation}
    |u^\top(h_{n+1}-h_n)|\leq\beta,\qquad 0\leq n<n_{\max}.
    \label{eq:residual-preservation}
\end{equation}
\end{assumption}

With $\widetilde h_n$ and $\kappa$ defined in \Cref{eq:evidence-residual}, these assumptions imply \Cref{thm:residual-count-readout} whenever $\kappa>0$.
The constructed decoder also satisfies, for every $\|\delta\|_2\leq\varepsilon$,
\begin{equation}
    |F(\widetilde h_n+\delta)-n|\leq\varepsilon/\kappa.
    \label{eq:stable-count-readout}
\end{equation}

\begin{proof}
Let $s=\operatorname{sign}(g)$ and project the residual states onto the unit direction $su$:
$t_n=su^\top \widetilde h_n$.
Using \Cref{eq:additive-evidence,eq:residual-preservation},
\begin{equation}
    \begin{aligned}
        t_{n+1}-t_n
        &=s u^\top(h_{n+1}-h_n)
          +|g|\bigl(c+\xi_{n+1}-\xi_n\bigr)\\
        &\geq-\beta+|g|(c-2\eta)=\kappa>0.
    \end{aligned}
    \label{eq:residual-projected-gap}
\end{equation}
Thus $t_0<\cdots<t_{n_{\max}}$.
For $m>n$, telescoping the increments yields
\begin{equation}
    \|\widetilde h_m-\widetilde h_n\|_2
    \geq |u^\top(\widetilde h_m-\widetilde h_n)|
    =t_m-t_n\geq\kappa(m-n).
\end{equation}
Symmetry gives \Cref{eq:residual-margin} for all pairs.

Define a continuous scalar function on the real line by
\begin{equation}
    f(t)=
    \begin{cases}
        0, & t\leq t_0,\\
        n+\dfrac{t-t_n}{t_{n+1}-t_n},
           & t_n\leq t\leq t_{n+1},\quad 0\leq n<n_{\max},\\
        n_{\max}, & t\geq t_{n_{\max}}.
    \end{cases}
    \label{eq:residual-interpolating-decoder}
\end{equation}
Each linear segment has slope between $0$ and $1/\kappa$, and the expressions agree at their endpoints.
Hence $f$ has global Lipschitz constant at most $1/\kappa$.
The decoder $F(x)=f(su^\top x)$ satisfies $F(\widetilde h_n)=n$ and
\begin{equation}
    |F(x)-F(y)|
    \leq \frac{|u^\top(x-y)|}{\kappa}
    \leq \frac{\|x-y\|_2}{\kappa}.
\end{equation}
Applying this inequality with $x=\widetilde h_n+\delta$ and $y=\widetilde h_n$ proves \Cref{eq:stable-count-readout}.
For $\varepsilon<\kappa/2$, the error is strictly less than $1/2$, so rounding to the nearest integer returns $n$.
\end{proof}

The theorem applies directly to residual addition in the original state space.
Its preservation condition bounds interference only along the shared evidence direction. Orthogonal changes in the base stream need not be small.
The same direction and constants must work across the entire stated count family.
For a collection with multiple images per count, coverage requires those images' representations to lie in the corresponding perturbation neighborhoods, or a separate uniform separation argument.
Neither this coverage nor the required approximate additivity is guaranteed by convolution alone.
The theorem establishes the existence of a stable readout at the injection representation, not that optimization finds it or that all later transformations preserve it.

A visual route can use information that is weak in the aggregated state.
In contrast, a residual depending only on that state, $\widehat h=h+gC(h)$ with $C$ having Lipschitz constant $L_C$, obeys
\begin{equation}
    \|\widehat h_{n+1}-\widehat h_n\|_2
    \leq(1+|g|L_C)\|h_{n+1}-h_n\|_2.
\end{equation}
Thus a residual with bounded sensitivity that acts only on an already compressed state has bounded distance amplification and cannot separate identical states.
Access to visual features retaining complementary count evidence is therefore relevant to the proposed route. Its usefulness remains an empirical question for the learned adapters.

\subsection{Object Formation and Aggregation Error}
\label{app:object-formation}

Let $\widehat{\mathcal O}(x,q)$ be a conceptual set of predicted object units,
$\hat m=|\widehat{\mathcal O}(x,q)|$, and $n=|\mathcal O(x,q)|$.
For a predicted count $\hat n$, the triangle inequality gives the valid decomposition
\begin{equation}
    \underbrace{|\hat n-n|}_{\mathcal E_{\mathrm{count}}}
    \leq
    \underbrace{|\hat m-n|}_{\mathcal E_{\mathrm{unit}}}
    +
    \underbrace{|\hat n-\hat m|}_{\mathcal E_{\mathrm{agg}}}.
    \label{eq:count-error-decomposition}
\end{equation}
Equality is not guaranteed because errors in the two stages can cancel.
Here $\mathcal E_{\mathrm{unit}}$ measures only the cardinality error of the formed units. It does not detect incorrect object identities or compensating splits and merges.
The decomposition separates two possible failure sources without asserting that one universally dominates the other.

For the patch fractions in \Cref{eq:patch-fragments}, define the fragmentation index
\begin{equation}
    J_k=1-\sum_{i=1}^M\rho_{ki}^2.
\end{equation}
Because the fractions are nonnegative and sum to one, $0\leq J_k\leq1-1/M$.
A disk contained in one patch has $J_k=0$. A disk split equally between two patches has $J_k=1/2$.
An unequal split between two patches has $J_k=2\rho(1-\rho)<1/2$.
This statistic describes how object support is distributed across tokens, without asserting a convergence rate or an ordering for arbitrary placements across multiple patches.

Patch occupancy is sufficient for counting only when each object occupies one patch and each occupied patch contains one object.
Even then, linear decodability requires a feature condition.
For example, suppose the initial embeddings after correcting for position
$\bar z_i=z_i^{(0)}-e_i$ obey a shared margin condition: there exist $u$, $t$, and $\gamma>0$ such that
\begin{equation}
    u^\top\bar z_i-t\geq\gamma
    \quad\text{on occupied patches},\qquad
    u^\top\bar z_i-t\leq-\gamma
    \quad\text{on background patches}.
\end{equation}
Then $\sum_i\mathbf 1\{u^\top\bar z_i>t\}=n$.
Patch alignment alone does not prove that an arbitrary learned projection satisfies this condition.

Objects with diverse appearances or objects crossing patch boundaries require grouping several fragments while separating nearby instances.
The pixel hierarchy used in the controlled ViT experiments has four stages, each consisting of a $3\times3$ convolution, ReLU, and $2\times2$ max pooling.
Each stage's features are projected with a $1\times1$ convolution, adaptively pooled to the ViT patch grid, normalized, and added to the patch tokens before the corresponding transformer block through a scalar gate.
The pretrained adapters follow \Cref{eq:token-convstack}. Their injection points and training setup are specified in \Cref{sec:experimental-setup}.
These architectural priors motivate the method. Neither exact additive evidence nor a strict improvement in information content is assumed.

\subsection{Metrics for Controlled Counting Experiments}
\label{app:counting-metrics}

For \Cref{tab:synpoly-main}, we evaluate each model on $T=1{,}000$ \synpoly{} images, with $100$ images for each count from $1$ to $10$.
Let $y_i\in\{1,\ldots,10\}$ be the true count of image $i$ and $p_i(k)$ its predicted probability for count $k$.
The controlled models obtain these probabilities by applying softmax to the scaled similarities between the image embedding and ten text prototypes, one for each count.
Each prototype is the normalized average of $20$ caption embeddings for that count.

\noindent\textbf{Accuracy.}
We use the most probable count class as the discrete prediction:
\begin{equation}
    \hat y_i^{\mathrm{argmax}}=\operatorname*{arg\,max}_{k\in\{1,\ldots,10\}}p_i(k),
    \qquad
    \mathrm{Accuracy}=\frac{1}{T}\sum_{i=1}^{T}
    \mathbf{1}\{\hat y_i^{\mathrm{argmax}}=y_i\}.
    \label{eq:controlled-count-accuracy}
\end{equation}
Accuracy is reported as a fraction in \Cref{tab:synpoly-main}.

\noindent\textbf{MAE and RMSE.}
Both error metrics use the continuous expected count from the same predictive distribution,
\begin{equation}
    \hat y_i^{\mathrm{exp}}=\sum_{k=1}^{10}k\,p_i(k),
    \label{eq:controlled-expected-count}
\end{equation}
without rounding:
\begin{equation}
    \mathrm{MAE}=\frac{1}{T}\sum_{i=1}^{T}|\hat y_i^{\mathrm{exp}}-y_i|,
    \qquad
    \mathrm{RMSE}=\sqrt{\frac{1}{T}\sum_{i=1}^{T}
    (\hat y_i^{\mathrm{exp}}-y_i)^2}.
    \label{eq:controlled-count-errors}
\end{equation}
Thus, accuracy measures exact classification, while MAE and RMSE measure the deviation of the expected count.

\noindent\textbf{Class separation.}
Let $z_i\in\mathbb R^{128}$ be the final image embedding, and let $\mathcal I_k=\{i:y_i=k\}$ index the images with count $k$.
We compute each class centroid and the mean Euclidean radius within that class as
\begin{equation}
    \mu_k=\frac{1}{|\mathcal I_k|}\sum_{i\in\mathcal I_k}z_i,
    \qquad
    \sigma_k=\frac{1}{|\mathcal I_k|}\sum_{i\in\mathcal I_k}
    \|z_i-\mu_k\|_2.
    \label{eq:controlled-class-radius}
\end{equation}
The reported score averages the normalized centroid distances over the nine adjacent count pairs:
\begin{equation}
    \mathrm{ClassSep}=\frac{1}{9}\sum_{k=1}^{9}
    \frac{\|\mu_{k+1}-\mu_k\|_2}{\sigma_k+\sigma_{k+1}+10^{-8}}.
    \label{eq:controlled-class-separation}
\end{equation}
A larger score indicates greater separation between adjacent count classes relative to the spread within each class.
All distances are computed in the original embedding space before t-SNE projection.

\subsection{Parameter Count}
\label{app:parameter-count}

\Cref{tab:parameter-count} summarizes the parameter budget: \methodname{} adds $2.708$M parameters and trains $42.827$M parameters including the existing visual merger.

\begin{table}[htbp]
    \centering
    \caption{\textbf{Parameter budget of \methodname{}-8B.} Percentages are relative to Qwen3-VL-8B's $8.767$B parameters, including vision and language components. The shared DeepStack adapter is counted once.}
    \label{tab:parameter-count}
    \setlength{\tabcolsep}{12pt}
    \begin{tabular}{@{}lrr@{}}
        \toprule
        Component & Parameters (M) & \% of base model \\
        \midrule
        Vision layer adapter & 0.596 & 0.007 \\
        Shared DeepStack adapter & 2.112 & 0.024 \\
        \textbf{Total added by \methodname{}} & \textbf{2.708} & \textbf{0.031} \\
        \midrule
        Existing visual merger & 40.119 & 0.458 \\
        \textbf{Total trainable} & \textbf{42.827} & \textbf{0.488} \\
        \bottomrule
    \end{tabular}
\end{table}

%% file: Sections/appendix_ablation.tex
\FloatBarrier
\section{Additional Ablation Study}
\label{app:ablation}

\subsection{ViT Injection Layer}
\label{app:vit-injection}

We vary the \methodname{} injection layer in Qwen3-VL-8B from $0$ to $26$, retaining adapters at all native DeepStack levels.
We select layer $18$ based on accuracy on the \textbf{PixMo-Count validation split}.
\Cref{fig:vit-injection-ablation} reports accuracy on PixMo-Count test and SpatialEval, where the selected layer also achieves the highest accuracy ($75.14\%$ and $66.67\%$, respectively).
On PixMo-Count test, it exceeds the runner up, layer $17$, by $1.71$ percentage points.

\begin{figure}[htbp]
    \centering
    \includegraphics[width=0.9\linewidth]{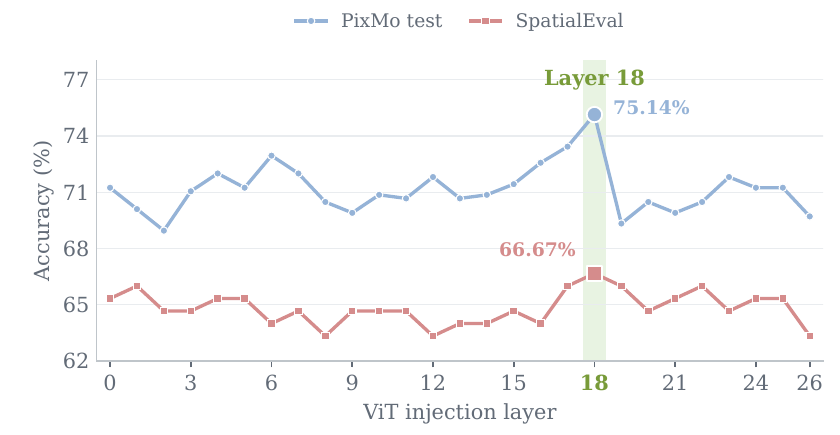}
    \caption{\textbf{Ablation of the ViT injection layer.}
    Accuracy on PixMo-Count test and SpatialEval for Qwen3-VL-8B with \methodname{} at each vision layer.
    The green band marks layer $18$, selected using PixMo-Count validation accuracy.}
    \label{fig:vit-injection-ablation}
\end{figure}

%% file: Sections/appendix_count_geometry.tex
\clearpage
\section{Quantitative Analysis of Count Representations}
\label{app:count-geometry}

We compare Qwen3-VL-8B and \methodname{} on the same $500$ \synpoly{} images used in \Cref{fig:synpoly-final-tsne}, with $50$ images per count from $1$ to $10$.
Both analyses use the final decoder state of the last prompt token before the final RMSNorm, computed from the image and counting prompt.
We measure separation in this original hidden space.

\subsection{Silhouette Across Count Classes}
\label{app:count-silhouette}

Using cosine distance, let $a_i$ be the mean distance from image $i$ to other images of the same count, and let $b_i$ be the smallest mean distance to any other count class.
The silhouette of image $i$ is
\begin{equation}
    s_i=\frac{b_i-a_i}{\max(a_i,b_i)}.
    \label{eq:prompt-count-silhouette}
\end{equation}
\Cref{fig:count-geometry}(a) reports the mean silhouette for each count.
\methodname{} improves the score in every class, raising the overall mean from $0.479$ to $0.707$.
Recomputing the metric within counts $6$ to $10$ gives an increase from $0.319$ to $0.627$, showing clearer separation among higher counts.

\subsection{Separation Between Adjacent Counts}
\label{app:adjacent-count-separation}

To measure how well neighboring counts are distinguished, we normalize each hidden vector to unit $\ell_2$ norm and compute the class centroid $\mu_n$ and mean distance to that centroid $\sigma_n$, as in \Cref{eq:controlled-class-radius}.
We report the centroid distance relative to the spread of the two classes:
\begin{equation}
    \Gamma_n=\frac{\|\mu_{n+1}-\mu_n\|_2}{\sigma_n+\sigma_{n+1}},
    \qquad n\in\{1,\ldots,9\}.
    \label{eq:prompt-adjacent-separation}
\end{equation}
As shown in \Cref{fig:count-geometry}(b), \methodname{} improves separation for all nine adjacent pairs.
The mean $\Gamma_n$ increases from $0.996$ to $1.449$. For counts $9$ and $10$, it increases from $0.558$ to $1.058$.
Together, these analyses show that the improved count clusters in \Cref{fig:synpoly-final-tsne} are accompanied by greater separation in the original representation space.

\begin{figure}[htbp]
    \centering
    \includegraphics[width=\linewidth]{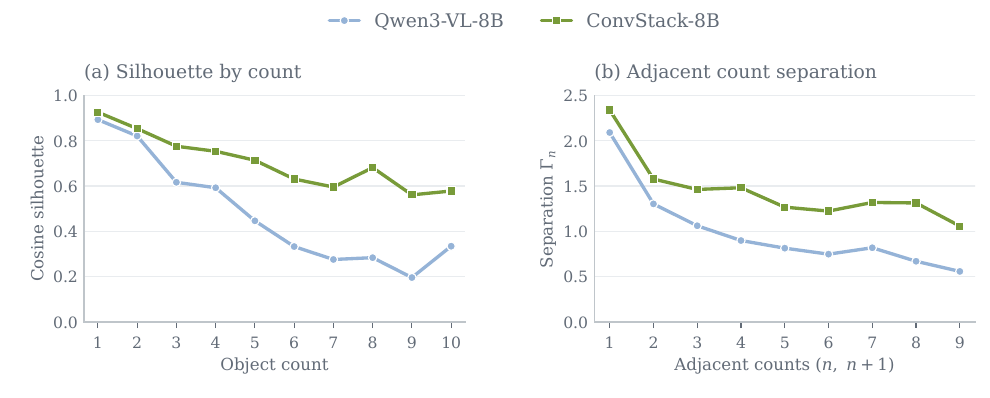}
    \caption{\textbf{Quantitative count separation on \synpoly{}.}
    (a) Mean cosine silhouette for each of the ten count classes.
    (b) Separation between counts $n$ and $n+1$, relative to their combined spread.
    Both models are evaluated on the same images. Larger values indicate better separation.}
    \label{fig:count-geometry}
\end{figure}